\documentclass{scrartcl}

\usepackage[utf8]{inputenc} % allow utf-8 input
\usepackage[T1]{fontenc}    % use 8-bit T1 fonts
\usepackage{amsfonts}       % blackboard math symbols
\usepackage{amsmath}
\usepackage{amssymb}
\usepackage{mathtools}
\usepackage{amsthm}
\usepackage{bm}
\usepackage{color}
\usepackage{soul}           % For strikethrough text.
\usepackage{url}
\usepackage{booktabs}       % professional-quality tables
\usepackage{nicefrac}       % compact symbols for 1/2, etc.
\usepackage{enumerate}
\usepackage{graphicx}
\DeclareGraphicsExtensions{.pdf,.png,.jpg}
\graphicspath{{figures/}}
\usepackage{comment}
\usepackage{hyperref}
\usepackage{centernot}     % For negating symbols

\usepackage{fullpage}
\usepackage[titletoc,title]{appendix}

\newcommand{\R}{\mathbb{R}}

\newcommand{\BigO}[1]{\ensuremath{\mathcal{O}\left(#1\right)}}                  % Big Oh.
\newcommand{\BigOm}[1]{\ensuremath{\Omega\left(#1\right)}}                      % Big Omega.
\newcommand{\poly}{\mathrm{poly}}
\newcommand{\vect}[1]{\ensuremath{\mathbf{#1}}}                                 % vector.
\newcommand{\mat}[1]{\ensuremath{\mathbf{\MakeUppercase{#1}}}}                  % Matrix for symbols.
\newcommand{\KL}[2]{\ensuremath{\mathbb{KL} \left(#1 \middle\Vert #2 \right)}}   % KL divergence.
\newcommand{\Exp}[2]{\ensuremath{\mathbb{E}_{#1}\left[#2\right]}}                % Expectation.
\newcommand{\Ind}[1]{\ensuremath{\mathbf{1}\left[#1\right]}}                     % Indicator function.
\newcommand{\Norm}[1]{\ensuremath{\left\lVert #1 \right\rVert}}                  % Norm
\newcommand{\NormI}[1]{\ensuremath{\left\lVert #1 \right\rVert}_1}               % L1 Norm 
\newcommand{\NormII}[1]{\ensuremath{\left\lVert #1 \right\rVert}_2}              % L2 Norm 
\newcommand{\NormInfty}[1]{\ensuremath{\left\lVert #1 \right\rVert_{\infty}}}    % Infinity Norm
\newcommand{\Diag}{\mathrm{\mathbf{Diag}}}                                       % Diagnonal matrix.
\newcommand{\InNorm}[1]{{\left\vert\kern-0.2ex\left\vert\kern-0.2ex\left\vert #1 
    \right\vert\kern-0.2ex\right\vert\kern-0.2ex\right\vert}}                    % Induced Norm.

\newcommand{\InNormII}[1]{{\left\vert\kern-0.2ex\left\vert\kern-0.2ex\left\vert #1 
    \right\vert\kern-0.2ex\right\vert\kern-0.2ex\right\vert}_2}                    % Induced 2 Norm (Spectral Norm).

\newcommand{\InNormInfty}[1]{{\left\vert\kern-0.2ex\left\vert\kern-0.2ex\left\vert #1 
    \right\vert\kern-0.2ex\right\vert\kern-0.2ex\right\vert}_{\infty}}           % Induced Infinity norm.

\newcommand{\Abs}[1]{\ensuremath{\left \lvert #1 \right \rvert}}                 % Absolute value.
\newcommand{\Prob}[1]{\ensuremath{\mathrm{Pr}\left\{ #1 \right\}}}               % Probability of event.
\newcommand{\iid}{i.i.d~}                                                        % IID.
\newcommand{\Grad}{\nabla}                                                       % Gradient symbol.
\DeclarePairedDelimiterX{\Inner}[2]{\langle}{\rangle}{#1, #2}                    % Inner product

\newcommand{\MI}{\mathnormal{I}}                                                     % Mutual information symbol.

\newcommand{\Land}{\wedge}                                                       % Logical And operator.
\newcommand{\defeq}{\overset{\mathrm{def}}{=}}                                   % Defnition equality

\newcommand{\csum}[2]{\sum_{\mathclap{#1}}^{\mathclap{#2}}}                      % Condensed sum superscripts and subscripts

\DeclareMathOperator*{\intersection}{\cap}

\DeclareMathOperator*{\argmin}{argmin}
\DeclareMathOperator*{\argmax}{argmax}

\newtheorem{proposition}{Proposition}
\newtheorem{assumption}{Assumption}
\newtheorem{lemma}{Lemma}
\newtheorem{theorem}{Theorem}
\newtheorem{remark}{Remark}

\newcommand{\Set}[1]{\{#1\}}
\newcommand{\x}{\mathbf{x}}
\newcommand{\w}{\mathbf{w}}
\newcommand{\win}{\ensuremath{\w_{-i}}} % w_{-i} w negative i
\newcommand{\wtin}{\ensuremath{{\w^*_{-i}}}} % w_{-i} w* negative i
\newcommand{\whin}{\ensuremath{\widehat{\w}_{-i}}} % w_{-i} w* negative i
\newcommand{\vv}{\ensuremath{\mathbf{v}}} % vector v
\newcommand{\vi}{\ensuremath{\mathbf{v}_i}} % v_i
\newcommand{\vhi}{\ensuremath{\mathbf{\widehat{v}}_i}} % v hat i
\newcommand{\vt}{\ensuremath{\mathbf{v^*}}} % true v
\newcommand{\vti}{\ensuremath{{\mathbf{v}^*_i}}} % true v_i
\newcommand{\vh}{\ensuremath{\widehat{\mathbf{v}}}} % v hat
\newcommand{\zv}{\ensuremath{\vect{z}}} % z vector
\newcommand{\zl}{\ensuremath{{\vect{z}^{(l)}}}} % z^l
\newcommand{\zi}{\ensuremath{\vect{z}_i}} % z_i
\newcommand{\zil}{\ensuremath{\vect{z}_i^{(l)}}} % z_i^l
\newcommand{\zsl}{\ensuremath{{\vect{z}^{(l)}_S}}} % z_S^l
\newcommand{\zs}{\ensuremath{\vect{z}_S}} % z_S
\newcommand{\xin}{\ensuremath{\x_{-i}}} % x_{-i} x negative i
\newcommand{\NE}{\ensuremath{\mathcal{NE}}}
\newcommand{\Data}{\ensuremath{\mathcal{D}}}
\renewcommand{\Game}{\ensuremath{\mathcal{G}}}
\newcommand{\Sc}{\ensuremath{S^c}} % Support complement.
\newcommand{\vs}{\ensuremath{\vect{v}_{S}}} % v_S (support)
\newcommand{\vhs}{\ensuremath{\vh_{S}}} % v hat S (support)
\newcommand{\vts}{\ensuremath{{\vect{v}^*_{S}}}} % v true S (support)
\newcommand{\loss}{\ensuremath{\ell}} % Loss function
\newcommand{\Hess}{\ensuremath{\nabla^2}} % Hessian
\newcommand{\ds}{\Delta_S} % difference vector
\newcommand{\dhs}{\widehat{\Delta}_S} % difference hat vector
\newcommand{\dbs}{\Delta_{\overline{S}}} % difference vector bar s
\newcommand{\dbsc}{\Delta_{\overline{\Sc}}} % difference vector bar sc
\newcommand{\yv}{\vect{y}} % y vector
\newcommand{\mumax}{\mu_{\mathrm{max}}}
\newcommand{\mumin}{\mu_{\mathrm{min}}}
\newcommand{\eigmin}{\lambda_{\mathrm{min}}}
\newcommand{\eigmax}{\lambda_{\mathrm{max}}}

\newcommand{\cmin}{C_{\mathrm{min}}}
\newcommand{\dmax}{D_{\mathrm{max}}}

\newcommand{\X}{\mathcal{X}} % Joint action space.
\newcommand{\U}{\mathcal{U}} % Set of payoffs of all players.
\newcommand{\rhomin}{\rho_{\mathrm{min}}}
\newcommand{\vtbs}{\vect{v}^*_{\overline{S}}}

\newcommand{\Pf}{\mathcal{P}}             % Probability distribution function.
\newcommand{\pmax}{p_{\mathrm{max}}}      % Maximum probability in PSNE set.
\newcommand{\pmin}{p_{\mathrm{min}}}      % Minimum probability in PSNE set.
\newcommand{\tpmax}{\tilde{p}_{\mathrm{max}}}      % Maximum probability in Non-PSNE set.
\newcommand{\tpmin}{\tilde{p}_{\mathrm{min}}}      % Minimum probability in Non-PSNE set.
\newcommand{\bZ}{\mat{Z}}
\newcommand{\bI}{\mat{I}}

\newcommand{\bP}{\mat{P}}

\newcommand{\by}{\vect{y}}
\newcommand{\cS}{\mathcal{S}}
\newcommand{\fG}{\mathfrak{G}}   % Game ensembles.
\newcommand{\ftG}{\widetilde{\mathfrak{G}}}   % restricted Game ensembles.
\newcommand{\perr}{p_{\mathrm{err}}}   % estimation error.

\newcommand{\fnet}{f_{\NE^*}}
\begin{document}

\title{Learning Graphical Games from Behavioral Data: Sufficient and Necessary Conditions}

\author{Asish Ghoshal and Jean Honorio\\
Department of Computer Science\\
Purdue University\\
West Lafayette, IN - 47906\\
\{aghoshal, jhonorio\}@purdue.edu}

\date{}

\maketitle

\begin{abstract}
In this paper we obtain sufficient and necessary conditions
on the number of samples required for exact recovery of the pure-strategy
Nash equilibria (PSNE) set of a graphical game 
from noisy observations of joint actions.
We consider sparse linear influence games ---
a parametric class of graphical games with linear payoffs, and
represented by directed graphs of $n$ nodes (players) and in-degree of at most $k$.
We show that one can efficiently recover the PSNE set of a linear influence game 
with $\BigO{k^2 \log n}$ samples, under very general observation models.
 On the other hand, we show that $\BigOm{k \log n}$
samples are necessary for any procedure to recover the PSNE set from observations
of joint actions. 
%and also on real-world data sets, where 
%we demonstrate the attractiveness of using graphical games to
%identify the most influential nations from United Nations voting records.
\end{abstract}
\section{Introduction and Related Work}
\label{sec:introduction}
Non-cooperative game theory is widely considered as an appropriate
mathematical framework for studying \emph{strategic} behavior in multi-agent scenarios.
In Non-cooperative game theory, the core solution concept of \emph{Nash equilibrium} 
describes the stable outcome of the overall behavior of self-interested agents  
--- for instance people, companies, governments, groups or autonomous systems ---
interacting strategically with each other and in distributed settings.

Over the past few years, considerable progress has been made in analyzing 
behavioral data using game-theoretic tools, e.g. 
computing Nash equilibria \cite{blum2006continuation,ortiz2002nash,vickrey2002multi},
most influential agents \cite{irfan14}, price of anarchy \cite{ben2011local} 
and related concepts in the context of graphical games.
In \emph{political science} for instance, Irfan and Ortiz \cite{irfan14} 
identified, from congressional voting records, the most influential senators in the U.S. congress --- 
a small set of senators whose collective behavior forces every other senator to a unique choice of vote.
Irfan and Ortiz \cite{irfan14} also observed that 
the most influential senators were strikingly similar to the gang-of-six senators, 
formed during the national debt ceiling negotiations of 2011.
Further, using graphical games,  Honorio and Ortiz \cite{honorio15} 
showed that Obama's influence on Republicans increased in the last sessions before candidacy, 
while McCain's influence on Republicans decreased.

The problems in \emph{algorithmic game theory} described above, i.e., 
computing the Nash equilibria, computing the price of anarchy or finding the most influential agents,
require a known graphical game which is not available apriori in real-world settings.
Therefore, Honorio and Ortiz \cite{honorio15} proposed learning graphical games from behavioral data,
using maximum likelihood estimation (MLE) and \emph{sparsity}-promoting methods.
On the other hand, Garg and Jaakkola \cite{Garg16} provide a discriminative 
approach to learn a class of graphical games called potential games.
Honorio and Ortiz \cite{honorio15} and Irfan and Ortiz 
\cite{irfan14} have also demonstrated the usefulness of learning \emph{sparse} graphical games from behavioral data
in real-world settings, through their analysis of
the voting records of the U.S. congress as well as the U.S. supreme court.

In this paper, we obtain necessary and sufficient conditions for recovering the PSNE set
of a graphical game in polynomial time. We also generalize the observation model from 
Ghoshal and Honorio \cite{ghoshal2016behavior}, to arbitrary distributions that satisfy certain
mild conditions. Our polynomial time method for
recovering the PSNE set, which was proposed by Honorio and Ortiz \cite{honorio15},
is based on using logistic regression for learning the neighborhood of each player in the graphical game, 
independently. Honorio and Ortiz \cite{honorio15}
showed that the method of independent logistic regression is likelihood consistent; i.e., in the
infinite sample limit, the likelihood estimate converges to the best achievable likelihood. In
this paper we obtain the stronger guarantee  of recovering the true PSNE set exactly. 

Finally, we would like to draw the attention of the reader to the fact that $\ell_1$-regularized logistic
regression has been analyzed by Ravikumar et. al. \cite{Ravikumar2010} in the context of learning
sparse Ising models. Apart from technical differences and differences in proof techniques,
our analysis of $\ell_1$-penalized logistic regression
for learning sparse graphical games differs from Ravikumar et. al. \cite{Ravikumar2010} conceptually --- in the sense
that we are not interested in recovering the edges of the true game graph, but only the PSNE set. Therefore,
we are able to avoid some stronger conditions required by Ravikumar et. al. \cite{Ravikumar2010}, 
such as mutual incoherence.

%The rest of the paper is organized as follows.
%In section \ref{sec:prelim} we provide a brief overview of graphical games and, specifically, linear influence games.
%We then formalize the problem of learning linear influence games in section \ref{sec:problem_formulation}.
%In section \ref{sec:results}, we present our main method and associated theoretical results. Then, in
%section \ref{sec:experiments} we provide experimental validation of our theoretical guarantees. Finally,
%in section \ref{sec:conclusion} we conclude by discussing avenues for future work.

\section{Preliminaries}
\label{sec:prelim}
In this section we provide some background information on graphical games introduced 
by Kearns et. al. \cite{kearns01}.
\subsection{Graphical Games}
A \emph{normal-form game} $\Game$ in classical game theory is defined by the triple
$\Game = (V, \X, \U)$ of players, actions and payoffs. $V$ is the set of players, 
and is given by the set $V = \{1,\ldots,n\}$, if there are $n$ players. $\X$
is the set of actions or \emph{pure-strategies} and is given by the Cartesian product
$\X \defeq \times_{i \in V} \X_i$, where $\X_i$ is the set of pure-strategies of the $i$-th player.
Finally, $\U \defeq \{u_i\}_{i=1}^n$, is the set of payoffs,
where $u_i: \X_i \times_{j \in V\setminus i} \X_j \rightarrow \R$ specifies the payoff for the $i$-th player given
its action and the joint actions of the all the remaining players. 

\sloppy
An important solution concept in the theory of non-cooperative games is that of \emph{Nash equilibrium}. 
For a non-cooperative game, a joint action $\x^* \in \X$ is a pure-strategy Nash equilibrium (PSNE)
if, for each player $i$, $x_i^* \in \argmax_{x_i \in \X_i} u_i(x_i, \x^*_{-i})$, where 
$\x^*_{-i} = \{x^*_j \vert j \neq i\}$. In other words,
$\x^*$ constitutes the mutual best-response for all players and no player has any incentive to 
unilaterally deviate from their optimal action $x^*_i$ given the joint actions of the remaining players $\x^*_{-i}$. 
The set of all \emph{pure-strategy Nash equilibrium} (PSNE) for a game $\Game$ is defined as follows:
\begin{align}
\NE(\Game) = \left\{\x^* \big| (\forall i \in V)\; x^*_i \in 
	\argmax_{x_i \in \X_i} u_i(x_i, \x^*_{-i}) \right\}. \label{eq_psne_set}
\end{align}
\fussy

Graphical games, introduced by Kearns et al. \cite{kearns01}, are game-theoretic analogues
of graphical models. A graphical game $G$ is defined by the \emph{directed graph}, $G = (V, E)$, of vertices and 
directed edges (arcs), where vertices correspond to players and arcs
encode ``influence'' among players, i.e., the payoff of  
the $i$-th player only depends on the actions of its (incoming) neighbors.

\subsection{Linear Influence Games}
Irfan and Ortiz \cite{irfan14} and Honorio and Ortiz \cite{honorio15}, introduced a specific form
of graphical games, called \emph{Linear Influential Games}, characterized by 
binary actions, or pure strategies, and linear payoff functions.
We assume, without loss of generality, that the joint action space $\X = \{-1, +1\}^n$.
A linear influence game between $n$ players, $\Game(n) = (\mat{W}, \vect{b})$, is characterized by
(i) a matrix of weights $\mat{W} \in \R^{n \times n}$, where the entry $w_{ij}$ 
indicates the amount of influence (signed) that the $j$-th player has on the $i$-th player and 
(ii) a bias vector $\vect{b} \in \R^n$, where $b_i$ captures
the prior preference of the $i$-th player for a particular action $x_i \in \{-1, +1\}$. 
The payoff of the $i$-th player is a linear function of the actions of the remaining players:
$u_i(x_i, \xin) = x_i(\win^T\xin - b_i) $, and the PSNE set is defined as follows:
\begin{align}
\NE(\Game(n)) = \left\{\x | (\forall i)\; x_i(\win^T\xin - b_i) \geq 0 \right\} \label{eq:lig_psne},
\end{align}
where $\win$ denotes the $i$-th row of $\mat{W}$ without the $i$-th entry, i.e.
$\win = \{w_{ij} \vert j \neq i \}$. Note that we have $\mathrm{diag}(\mat{W}) = 0$.
%For linear influence games $\Game(n)$, we can
%define the neighborhood and signed neighborhood of the $i$-th vertex as 
%$\Nh(i) = \{j | \Abs{w_{i,j}} > 0\}$ and $\Nh_{\pm}(i) = \{\sign(w_{ij}) | j \in \Nh(i)\}$
%respectively.
%It is important to note that we don't include the outgoing arcs in our definition of
%the neighborhood of a vertex.
Thus, for linear influence games, the weight matrix $\mat{W}$ and 
the bias vector $\vect{b}$, completely specify the game and the PSNE set
induced by the game. Finally, let $\Game(n, k)$ denote a sparse game
over $n$ players where the in-degree of any vertex is at most $k$.
\section{Problem Formulation}
\label{sec:problem_formulation}
Now we turn our attention to the problem of learning graphical games from
observations of joint actions only. Let $\NE^* \defeq \NE(\Game^*(n, k))$. We assume that
there exists a game $\Game^*(n, k) = (\mat{W}^*, \vect{b}^*)$ from which a ``noisy'' 
data set $\Data = \{\x^{(l)}\}_{l=1}^m$ of $m$ observations is generated,
where each observation $\x^{(l)}$ is sampled independently and
identically from some distribution $\Pf$. 
We will use two specific distributions $\Pf_g$ and $\Pf_l$, which we 
refer to as the \emph{global} and \emph{local} noise model, to provide
further intuition behind our results. In the global noise model,
we assume that a joint action is observed from the PSNE set with probability
$q_g \in (\nicefrac{\Abs{\NE^*}}{2^n}, 1)$, i.e.
\begin{align}
\Pf_g(\x) = \frac{q_g \Ind{\x \in \NE^*}}{|\NE^*|} 
	+ \frac{(1 - q_g)\Ind{\x \notin \NE^*}}{2^n - |\NE^*|} \label{eq:obs_model_global}.
\end{align}
In the above distribution, $q_g$ can be thought of as the ``signal'' level in the data set,
while $1 - q_g$ can be thought of as the ``noise'' level in the data set. 
In the local noise model we assume that the joint actions are drawn from the PSNE set
with the action of each player corrupted independently by some Bernoulli noise.
Then in the local noise model the distribution over
joint actions is given as follows:
\begin{align}
\Pf_l(\x) = \frac{1}{\Abs{\NE^*}} \sum_{\by \in \NE^*}
	 \prod_{i=1}^n q_i^{\Ind{x_i = y_i}} (1 - q_i)^{\Ind{x_i \neq y_i}},
\label{eq:obs_model_local}
\end{align}
where $q_i > 0.5$. While these noise models were introduced in \cite{honorio15}, we obtain
our results with respect to very general observation models, satisfying only some mild conditions.
%We further assume that the game is 
%non-trivial\footnote{This comes from Definition 4 in \cite{honorio15}.}
%i.e. $\Abs{\NE^*} \in \{1, \ldots 2^n - 1\}$ and that $q \in (\nicefrac{\Abs{\NE^*}}{2^n}, 1)$.
%The latter assumption ensures that the signal level in the data set is more than the noise level
%\footnote{See Proposition 5 and Definition 7 in \cite{honorio15} for a justification of this.}.
%We define the equality of two games $\Game^*(n, k) = (\mat{W}^*, \vect{b}^*)$
%and $\widehat{\Game}(n, k) = (\widehat{\mat{W}}, \widehat{\vect{b}})$ as follows:
%\begin{gather*}
%\Game^*(n, k) = \widehat{\Game}(n, k) \text{ iff } \\
%(\forall i)\; \Nh^*_{\pm}(i) = \widehat{\Nh}_{\pm}(i)\; \Land 
%\sign(b^*_i) = \sign(\widehat{b}_i)\; \Land \; \NE^* = \widehat{\NE}.
%\end{gather*}
A natural question to ask then is that: ``Given only the data set $\Data$ and no other
information, is it possible to recover the game graph?''
Honorio and Ortiz \cite{honorio15} showed that 
it is in general impossible to learn the true game $\Game^*(n, k)$ from observations of joint actions only
because multiple weight matrices $\mat{W}$ and bias vectors $\vect{b}$ can induce the same 
PSNE set and therefore have the same likelihood under the global noise model \eqref{eq:obs_model_global} ---
an issue known as non-identifiablity in the statistics literature. It is also easy to see that
the same holds true for the local noise model. It is, however, possible to learn the equivalence class of games
that induce the same PSNE set. We define the equivalence of two games $\Game^*(n, k)$
and $\widehat{\Game}(n, k)$ simply as :
\begin{gather*}
\Game^*(n, k) \equiv \widehat{\Game}(n, k) \text{ iff } \NE(\Game^*(n, k)) = \NE(\widehat{\Game}(n, k)).
\end{gather*}
Therefore, our goal in this paper is efficient 
and consistent recovery of the pure-strategy Nash equilibria set (PSNE) from observations
of joint actions only; i.e., given a data set $\Data$, drawn from some game $\Game^*(n, k)$
according to the distribution $\Pf$, we infer a game $\widehat{\Game}(n, k)$ from $\Data$
such that $\widehat{\Game}(n, k) \equiv \Game^*(n, k)$.
\section{Method and Results}
\label{sec:results}
In order to efficiently learn games, we make a few assumptions on the probability distribution
from which samples are drawn and also on the underlying game.
\subsection{Assumptions}
The following assumption ensures that the distribution $\Pf$ assigns non-zero
mass to all joint actions in $\X$ and that the signal level in the data set
is more than the noise level.
\begin{assumption}
\label{ass:distribution}
There exists constants $\tpmin, \tpmax$ and $\pmax$
such that the data distribution $\Pf$ satisfies the following:
\begin{gather*}
0 < \frac{\tpmin}{2^n - \Abs{\NE^*}} \leq \Pf(\x) \leq \frac{\tpmax}{2^n - \Abs{\NE^*}}, \forall \x \in \X \setminus \NE^*, \\
\frac{\tpmax}{2^n - \Abs{\NE^*}} < \Pf(\x) \leq \pmax \leq 1,\, \forall \x \in \NE^*.
\end{gather*}
\end{assumption}
To get some intuition for the above assumption, consider the global noise model. In
this case we have that $\tpmin = \tpmax = (1 - q_g)$, $\pmax = \nicefrac{q_g}{\Abs{\NE^*}}$, and
$\forall \x \in \NE^*,\, \Pf(\x) = \pmax$. For the local noise model, consider, for simplicity,
the case when there are only two joint actions in the PSNE set:
 $\NE^* = \Set{\x^1, \x^2}$, such that $x^1_1 = +1, x^2_1 = -1$ and $x^1_i = x^2_i = +1$
for all $i \neq 1$. Then, $\tpmin = 0.5 \times (1 - q_2) \times \ldots \times (1 - q_n) \times (2^n - 2)$,
$\tpmax = 0.5 \times (1 - q_j) (\prod_{i\notin \Set{j,1}} q_i) \times (2^n - 2)$, where $q_j = \min\Set{q_2, \ldots, q_n}$,
and $\pmax = 0.5 \times q_2 \times \ldots q_n$.

Our next assumption concerns with the minimum payoff in the PSNE set.
\begin{assumption}
\label{ass:payoff}
The minimum payoff in the PSNE set, $\rhomin$, is strictly positive, specifically:
\begin{align*}
x_i(\wtin^T \xin - b_i) \geq \rhomin > \nicefrac{5 \cmin}{\dmax}
                && (\forall\; \vect{x} \in \NE^*),
\end{align*}
where $\cmin > 0$ and $\dmax$ are the minimum and maximum eigenvalue 
of the expected Hessian and scatter matrices respectively.
\end{assumption}
Note that as long as the minimum payoff is strictly positive, we
can scale the parameters $(\mat{W}^*, \vect{b}^*)$ by the constant
$\nicefrac{5 \cmin}{\dmax}$ to satisfy the condition: $\rhomin > \nicefrac{5 \cmin}{\dmax}$,
without changing the PSNE set. Indeed the assumption that the minimum payoff is strictly positive is
is unavoidable for exact recovery of the PSNE set in a 
noisy setting such as ours, because otherwise this is akin to exactly recovering
the parameters $\vv$ for each player $i$. For example, if $\x \in \NE^*$ is such that
$\vt^T \x = 0$, then it can be shown that even if $\NormInfty{\vt - \vh} = \varepsilon$,
for any $\varepsilon$ arbitrarily close to $0$, then $\vh^T \x < 0$ and therefore 
$\NE(\mat{W}^*, \vect{b}^*) \neq \NE(\widehat{\mat{W}}, \widehat{\vect{b}})$.

\subsection{Method}
Our main method for learning the structure of a sparse LIG, $\Game^*(n,k)$,
is based on using $\ell_1$-regularized logistic regression, to learn
the parameters $(\win, b_i)$ for each player $i$ independently.
We denote by $\vi(\mat{W}, \vect{b}) = (\win, -b_i)$ the parameter vector 
for the $i$-th player, which characterizes its
payoff; by $\zi(\x) = (x_i \xin,~ x_i)$ the ``feature'' vector. In the rest of the paper we
use $\vi$ and $\zi$ instead of $\vi(\mat{W}, \vect{b})$ and $\zi(\x)$ respectively, to simplify notation.
Then, we learn the parameters for the $i$-th player as follows:
\begin{align}
	\vhi &= \argmin_{\vi} \loss(\vi, \Data) + \lambda \NormI{\vi} \label{eq:optimization} \\
	\loss(\vi, \Data) &= \frac{1}{m} \sum_{l=1}^m \log(1 + \exp(-\vi^T\zil)) \label{eq:loss}.
\end{align}
We then set $\whin = [\vhi]_{1:(n - 1)}$ and $\widehat{b}_i = -[\vhi]_n$, where the notation
$[.]_{i:j}$ denotes indices $i$ to $j$ of the vector.  
We take a moment to introduce the expressions of the gradient and the Hessian of the 
loss function \eqref{eq:loss}, which will be useful later. The gradient and Hessian of the loss
function for any vector $\vv$ and the data set $\Data$ is given as follows:
\begin{align}
\Grad \loss(\vect{v}, \Data) = \frac{1}{m} 
	\sum_{l=1}^m\left\{ \frac{- \zl}{1 + \exp(\vect{v}^T \zl)} \right\} \label{eq:grad} \\
\Hess \loss(\vect{v}, \Data) = \frac{1}{m}  	
	\sum _{l=1}^m \eta(\vv^T \zl) \zl \zl^T,
\end{align}
where $\eta(x) = \nicefrac{1}{(e^{x/2} + e^{-x/2})^2}$. Finally, $\mat{H}^m_i$ denotes
the sample Hessian matrix with respect to the $i$-th player and 
the true parameter $\vti$, and  $\mat{H}_i^*$ denotes its expected value,
i.e. $\mat{H}_i^* \defeq \Exp{\Data}{\mat{H}^m_i} = 
\Exp{\Data}{\Hess \loss(\vti, \Data)}$.
In subsequent sections we drop the notational dependence of $\mat{H}_i^*$ and $\zi$ on $i$ to simplify notation.

We show that, under the aforementioned assumptions on the true game $\Game^*(n,k) = (\mat{W}^*, \vect{b}^*)$,
the parameters $\widehat{\mat{W}}$ and $\widehat{\vect{b}}$ obtained using
\eqref{eq:loss} induce the same PSNE set as the true game, i.e., $\NE(\mat{W}^*, \vect{b}^*) = 
\NE(\widehat{\mat{W}}, \widehat{\vect{b}})$.

\subsection{Sufficient Conditions}
In this section, we derive sufficient conditions on the number of samples for efficiently
recovering the PSNE set of graphical games with linear payoffs.
To start with, we make the following observation regarding the number of Nash equilibria of the
game satisfying Assumption \ref{ass:payoff}. The proof of the following proposition,
as well as other missing proofs can be found in Appendix \ref{app:detailed_proofs}.
\begin{proposition}
\label{prop:number_of_equilibria}
The number of Nash equilibria of a non-trivial game ($\Abs{\NE^*} \in [1, 2^n - 1]$) 
satisfying Assumption \ref{ass:payoff}
is at most $2^{n-1}$.
\end{proposition}
We will denote the fraction of joint actions that are in the PSNE set by $\fnet \defeq \nicefrac{\Abs{\NE^*}}{2^{n-1}}$.
By proposition \ref{prop:number_of_equilibria}, $\fnet \in (0, 1]$.
Then, our main strategy for obtaining sufficient conditions for exact PSNE recovery guarantees is to first show
that under any data distribution $\Pf$ that satisfies Assumption \ref{ass:distribution},
the expected loss is smooth and strongly convex, i.e., the population Hessian matrix is
positive definite and the population scatter matrix has eigenvalues bounded by a constant.
Then using tools from random matrix theory, we
show that the sample Hessian and scatter matrices are ``well behaved'', i.e., are positive definite
and have bounded eigenvalues respectively, with high probability.
Then, we exploit the convexity properties of the logistic loss function to show that 
the weight vectors learned using penalized logistic regression are ``close'' to the true weight vectors. 
By our assumption that the minimum payoff in the PSNE set is strictly greater than zero, 
we show that the weight vectors inferred from a finite sample of joint actions
induce the same PSNE set as the true weight vectors.

The following lemma shows that the expected Hessian matrices for each player is positive definite
and the maximum eigenvalues of the expected scatter matrices are bounded from above by a constant.
\begin{lemma}
Let $S$ be the support of the vector $\vv$, i.e., $S \defeq \{i \vert \Abs{v_i} > 0 \}$.
There exists constant $\cmin \geq \frac{\eta(\NormI{\vt}) 2^n \tpmin}{2^n - \Abs{\NE^*}} > 0 $
and $\dmax \leq 2^n \pmax$, such that we have
$\eigmin(\mat{H}_{SS}^*) = \cmin$ and $\eigmax(\Exp{\x}{\zv_S \zv_S^T}) = \dmax$.
\end{lemma}
\begin{proof}
\begin{align*}
\eigmin(\mat{H}_{SS}^*) &= \eigmin\Bigl(\Exp{\x}{\eta(\vt^T \zv) \zv_S \zv_S^T}\Bigr) \\
	&= \eta(\NormI{\vt}) \eigmin(\Exp{\x}{\zv_S \zv_S^T}).
\end{align*}
Let $\bZ \defeq \Set{\zv_S | \x \in \X}$ and
 $\bP \defeq \Diag((\Pf(\x))_{\x \in \X})$,
 where $\zv_S$ denotes the feature vector for the $i$-th player constrained to the support
set $S$ for some $i$.
Note that $\bZ  \in \Set{-1, 1}^{2^n \times \Abs{S}}$; $\bP \in \R^{2^n \times 2^n}$
and is positive definite by our assumption that  the minimum probability $\frac{\tpmin}{2^n - \Abs{\NE^*}} > 0$. 
Further note that the columns of $\bZ$ are orthogonal and $\bZ^T \bZ = 2^n \bI_{\Abs{S}}$,
where $\bI_{\Abs{S}}$ is the $\Abs{S} \times \Abs{S}$ identity matrix.
Then we have that
\begin{align*}
&\eigmin(\Exp{\x}{\zv_S \zv_S^T}) = \min_{\Set{\by \in \R^{\Abs{S}} | \NormII{\by} = 1}} \by^T \bZ^T \bP \bZ \by \\
&\quad= \min_{\Set{\by' \in \R^{2^n} | \by' = \nicefrac{\bZ \by}{\sqrt{2^n}} \Land \by \in \R^{\Abs{S}} \Land \NormII{\by}=1}}
	 2^n (\by')^T \bP \by' \\
&\quad \geq \min_{\Set{\by' \in \R^{2^n} | \NormII{\by'} = 1}} 2^n (\by')^T \bP \by'  \\
&=\quad 2^n \eigmin(\bP) = \frac{2^n \tpmin}{2^n - \Abs{\NE^*}}
\end{align*}
Therefore, the minimum eigenvalue of $\mat{H}_{SS}^*$ is lower bounded as follows:
\begin{align*}
\eigmin(\mat{H}_{SS}^*) = \cmin \geq \frac{\eta(\NormI{\vt}) 2^n \pmin}{2^n - \Abs{\NE^*}} > 0.
\end{align*}
Similarly, the maximum eigenvalue of $\Exp{\x}{\zv_S \zv_S^T}$ can be bounded as
$\eigmax(\Exp{\x}{\zv_S \zv_S^T}) = \eigmax(\bZ^T \bP \bZ) \leq  2^n \pmax$.
\end{proof}
\subsubsection{Minimum and Maximum Eigenvalues of Finite Sample Hessian and Scatter Matrices}
The following technical lemma shows that the eigenvalues conditions of the expected Hessian
and scatter matrices, hold with high probability in the finite sample case.
\begin{lemma}
\label{lemma:eigvalue}
If $\eigmin(\mat{H}^*_{SS}) \geq \cmin,\, \eigmax(\Exp{\x}{\zv_S \zv_S^T}) \leq \dmax$,
then we have that 
\begin{gather*}
\eigmin(\mat{H}^m_{SS}) \geq \frac{\cmin}{2},\,
\eigmax\left(\sum_{l=1}^m \zsl \zsl^T \right) \leq 2 \dmax
\end{gather*}
with probability at least 
\begin{gather*}
1 - \Abs{S} \exp{\left(\frac{-m \cmin}{2 \Abs{S}}\right)} \text{ and } 
1 - \Abs{S} \exp\left( - \frac{m \tpmin}{4\Abs{S}} \right)
\end{gather*}
respectively.
\end{lemma}
\begin{proof}
Let 
\begin{align*}
\mumin \defeq \eigmin(\mat{H}^*_{SS}) \text{ and } \mumax \defeq \eigmax(\Exp{\x}{\zv_S \zv_S^T}).
\end{align*}
First note that for all $\zv \in \{-1, +1\}^n$:
\begin{gather*}
\eigmax(\eta(\vts^T \zs) \zs \zs^T) \leq \frac{\Abs{S}}{4} \defeq R \\
\eigmax(\zs \zs^T) \leq \Abs{S} \defeq R'.
\end{gather*}
Using the Matrix Chernoff bounds from Tropp \cite{tropp2012user}, we have that 
\begin{align*}
\Prob{\eigmin(\mat{H}^m_{SS}) \leq (1 - \delta) \mumin} \leq
	 \Abs{S} \left( \frac{e^{-\delta}}{(1 - \delta)^{1 - \delta}} \right)^{\mathclap{\frac{m\mumin}{R}}}.
\end{align*}
Setting $\delta = \nicefrac{1}{2}$ we get that
\begin{align*}
\Prob{\eigmin(\mat{H}^m_{SS}) \leq \nicefrac{\mumin}{2}} 
	&\leq  \Abs{S} \left[\sqrt{\frac{2}{e}}\right]^{\frac{4 m \mumin}{\Abs{S}}} \\
	&\leq \Abs{S} \exp \left(\frac{-m\cmin}{2 \Abs{S}}\right).
\end{align*}
Therefore, we have
\begin{align*}
\Prob{\eigmin(\mat{H}^m_{SS}) > \nicefrac{\cmin}{2}} > 
	1 - \Abs{S} \exp \left(\frac{-m\cmin}{2 \Abs{S}}\right).
\end{align*}
Next, we have that
\begin{align*}
\mumax &= \eigmax(\Exp{\x}{\zv_S \zv_S^T}) \\
&\geq \eigmin(\Exp{\x}{\zv_S \zv_S^T}) 
\geq \frac{2^n \tpmin}{2^n - \Abs{\NE^*}} 
\end{align*}
Once again invoking Theorem 1.1 from \cite{tropp2012user} and setting $\delta = 1$ we have that
\begin{align*}
\Prob{\eigmax \geq (1 + \delta)\mumax} 
&\leq \Abs{S} \left[ \frac{e^\delta}{(1 + \delta)^{1 + \delta}}\right]^{\nicefrac{(m \mumax)}{R'}} \\
\Prob{\eigmax \geq 2 \mumax} 
&\leq \textstyle{\Abs{S} \left[ \frac{e}{4}\right]^{\nicefrac{(m \mumax)}{\Abs{S}}}} \\
&\leq \textstyle{\Abs{S} \exp\left( \frac{- m \mumax}{4\Abs{S}} \right)} \\
&\leq \textstyle{\Abs{S} \exp\left( - \frac{m 2^{n-2} \tpmin}{\Abs{S}(2^n - \Abs{\NE^*})} \right)} \\
&\leq \textstyle{\Abs{S} \exp\left(\frac{- m \tpmin}{4\Abs{S}} \right)}
\end{align*}
Therefore, we have that
\begin{align*}
\Prob{\eigmax < 2 \dmax} > 1 - \Abs{S} \exp\left( - \frac{m \tpmin}{4\Abs{S}} \right).
\end{align*}
\end{proof}

\subsubsection{Recovering the Pure Strategy Nash Equilibria (PSNE) Set}
Before presenting our main result on the exact recovery of the PSNE set from noisy
observations of joint actions, we first present a few technical lemmas that would be
helpful in proving the main result.
The following lemma bounds the gradient of the loss function \eqref{eq:loss}
at the true vector $\vt$, for all players.
\begin{lemma}
\label{lemma:grad_bound}
With probability at least $1 - \delta$ for $\delta \in (0, 1)$, we have that
\begin{align*}
\NormInfty{\Grad \loss(\vt, \Data)} < \nu + \sqrt{\frac{2}{m} \log \frac{2n}{\delta}},
\end{align*}
where $\kappa = \nicefrac{1}{(1 + \exp(\rhomin))}$, $\rhomin \geq 0$ is the minimum 
payoff in the PSNE set, $\fnet = \nicefrac{\Abs{\NE^*}}{2^{n-1}}$, and
\begin{align}
\nu \defeq \kappa \csum{\x \in \NE^*}{} \Pf(x) + 
	\frac{(\tpmax - \tpmin)}{2 - \fnet} + \frac{\fnet \tpmin}{2 - \fnet} \label{eq:nu}
\end{align}
\end{lemma}
\begin{proof}
Consider the $i$-th player. Let $\vect{u}^m \defeq \Grad \loss(\vv_i^*, \Data)$ and  $u^m_j$ denote the $j$-th index of
$\vect{u}^m$. For any subset $\cS' \subset \X$ such that $\Abs{\cS'} = 2^{n-1}$ define
the function $g(\cS')$ as follows:
\begin{align*}
g(\cS') \defeq \sum_{\x \in \cS'} \Pf(\x) f(\x) - \sum_{\x \in \cS'^c} \Pf(\x) f(\x),
\end{align*}
where $\cS'^c$ denotes the complement of the set $\cS'$ and $f(\x) = \nicefrac{1}{1 + \exp({\vv_i^*}^T \zv_i(\x))}$.
For $\x \in \NE^*$, $f(\x) \leq \kappa$, while for $\x \notin \NE^*$ we have $\nicefrac{1}{2} \leq f(\x) \leq 1$.
Lastly, let $\cS_{ij} = \Set{\x \in \X | x_i x_j = +1}$ and
$\cS_{i} = \Set{\x \in \X | x_i = +1}$.
From \eqref{eq:grad} we have that, 
for $j \neq n$, $\Abs{\Exp{}{u^m_j}} = \Abs{g(\cS_{ij})}$, while for $j=n$
$\Abs{\Exp{}{u^m_j}} = \Abs{g(\cS_{i})}$. Thus we get
\begin{align}
\NormInfty{\vect{u}^m} \leq \max_{\cS' \subset \X| \Abs{\cS'} = 2^{n - 1}} g(\cS') \label{eq:grad1}
\end{align}
Let $\cS$ be the set that maximizes \eqref{eq:grad1},
$A \defeq \cS \intersection \NE^*$ and $B \defeq \cS^c \intersection \NE^*$.
Continuing from above, 
\begin{align*}
\Abs{g(\cS)} &= \Bigl| \csum{\x \in \cS \setminus A}{} \Pf(\x) f(\x) + \csum{\x \in A}{} \Pf(\x) f(\x) \\
	&\quad - \csum{\x \in \cS^c\setminus B}{} \Pf(\x) f(\x) - \csum{\x \in B}{} \Pf(\x) f(\x) \Bigr| \\
&\leq \kappa \csum{\x \in \NE^*}{} \Pf(x)  + 
	\Bigl| \csum{\x \in \cS \setminus A}{} \Pf(\x) f(\x) - \csum{\x \in \cS^c\setminus B}{} \Pf(\x) f(\x) \Bigr| 
\end{align*}
Assume that the first term inside the absolute value above dominates the second term, if not then we 
can proceed by reversing the two terms.
\begin{align*}
\Abs{g(\cS)} &\leq \kappa \csum{\x \in \NE^*}{} \Pf(x) 
	+ \frac{2^{n - 1} \tpmax - (2^{n - 1} - \Abs{\NE^*})\tpmin}{2^n - \Abs{\NE^*}} \\
&= \kappa \csum{\mathclap{\x \in \NE^*}}{} \Pf(x) + 
	\frac{(\tpmax - \tpmin) + \fnet \tpmin}{2 - \fnet} = \nu	
\end{align*}
Also note that $\Abs{u^m_j} \leq \nu \leq 1$. Finally, from Hoeffding's inequality \cite{hoeffding1963probability}
and using a union bound argument over all players, we have that:
\begin{align*}
&\Prob{\max_{j=1}^n \Abs{u^m_j - \Exp{}{u^m_j}} < t}  > 1 - 2ne^{\nicefrac{-mt^2}{2}} \\
&\implies \Prob{\NormInfty{\vect{u^m} - \Exp{}{\vect{u^m}}} < t}  > 1 - 2ne^{\nicefrac{-mt^2}{2}} \\
&\implies \Prob{\NormInfty{\vect{u^m}} - \NormInfty{\Exp{}{\vect{u^m}}} < t}  > 1 - 2ne^{\nicefrac{-mt^2}{2}} \\
&\implies \Prob{\NormInfty{\vect{u^m}} < \nu + t} > 1 - 2ne^{\nicefrac{-mt^2}{2}}.
\end{align*}
Setting $2n\exp(\nicefrac{-mt^2}{2}) = \delta$, we prove our claim.
\end{proof}
To get some intuition for the lemma above, consider the constant $\nu$ as given in \eqref{eq:nu}.
First, note that $\kappa \leq \nicefrac{1}{2}$. Also, as 
the minimum payoff $\rhomin$ increases, $\kappa$ decays to $0$ exponentially.
Similarly, if the probability measure
on the non-Nash equilibria set is close to uniform, meaning $\tpmax - \tpmin \approx 0$,
then the second term in \eqref{eq:nu} vanishes. Finally, if the fraction of actions that
are in the PSNE set ($\fnet$) is small, then the third term in \eqref{eq:nu} is small.
Therefore, if the minimum payoff is high, the noise distribution, i.e., the distribution of
the non-Nash equilibria joint actions, is close to uniform, and the fraction of joint actions that
are in the PSNE set is small, then the expected gradient vanishes.
In the following technical lemma we show that the optimal vector $\vh$
for the logistic regression problem is close to the true vector $\vt$ in the support set $S$ of $\vt$. 
Next, in Lemma \ref{lemma_l1norm_bound}, we bound the difference between the
true vector $\vt$ and the optimal vector $\vh$ in the non-support set. The lemmas
together show that the optimal vector is close to the true vector. 
\begin{lemma}
\label{lemma_l2norm_bound}
If the regularization parameter $\lambda$ 
satisfies the following condition:
\begin{gather*}
\lambda \leq \frac{5\cmin^2}{16 \Abs{S} \dmax} - \nu - \sqrt{\frac{2}{m} \log \frac{2n}{\delta}},
\end{gather*}
then 
\begin{align*}
\NormII{\vts - \vhs} \leq \frac{5 \cmin}{4 \sqrt{\Abs{S}} \dmax},
\end{align*}
with probability at least $1 - (\delta + \Abs{S} \exp (\nicefrac{(-m \cmin)}{2 \Abs{S}}) + 
\Abs{S} \exp ( - \nicefrac{m \tpmin}{4\Abs{S}} ) )$.
\end{lemma}
\begin{proof}
The proof of this lemma follows the general proof structure of Lemma 3 in \cite{Ravikumar2010}.
First, we reparameterize the $\ell_1$-regularized loss function
\begin{align*}
f(\vect{v}_S) = \loss(\vs) + \lambda \NormI{\vs}
\end{align*}
as the loss function $\widetilde{f}$, which gives the loss at a point
that is $\ds$ distance away from the true parameter $\vts$ as :
$\widetilde{f}(\ds) = \loss(\vts + \ds) - \loss(\vts)
    + \lambda (\NormI{\vts + \ds} - \NormI{\vts})$,
where $\ds = \vect{v}_S - \vts$.
Also note that the loss function $\widetilde{f}$ is shifted such that
the loss at the true parameter $\vts$ is $0$, i.e., $\widetilde{f}(\vect{0}) = 0$.
Further, note that the function
$\widetilde{f}$ is convex and is minimized at $\dhs = \vhs - \vts$,
since $\vhs$ minimizes $f$. Therefore, clearly $\widetilde{f}(\dhs) \leq 0$.
Thus, if we can show that the function $\widetilde{f}$ is strictly positive
on the surface of a ball of radius $b$, then the point $\dhs$ lies
inside the ball i.e., $\Norm{\vhs - \vts}_2 \leq b$. Using the Taylor's theorem
we expand the first term of $\widetilde{f}$ to get the following:
\begin{align}
\widetilde{f}(\ds) &= \Grad \loss(\vts)^T \ds + \ds^T \Hess \loss(\vts + \theta \ds) \ds \notag \\
   &\qquad+ \lambda (\NormI{\vts + \ds} - \NormI{\vts}), \label{eq:f}
\end{align}
for some $\theta \in [0, 1]$. Next, we lower bound each of the terms in \eqref{eq:f}.
Using the Cauchy-Schwartz inequality, the first term in \eqref{eq:f} is bounded as follows:
\begin{align}
\Grad\loss(\vts)^T\ds &\geq - \NormInfty{\Grad\loss(\vts)} \NormI{\ds} \notag \\
&\geq - \NormInfty{\Grad\loss(\vts)} \sqrt{|S|} \NormII{\ds} \notag \\
&\geq - b \sqrt{|S|} \left(\nu + \sqrt{\frac{2}{m} \log \frac{2n}{\delta}}\right), \label{eq:lb1}
\end{align}
with probability at least $1 - \delta$ for $\delta \in [0, 1]$.
It is also easy to upper bound the last term in equation \ref{eq:f},
using the reverse triangle inequality as follows:
\begin{align*}
\lambda \Abs{\NormI{\vts + \ds} - \NormI{\vts}} \leq \lambda \NormI{\ds}.
\end{align*}
Which then implies the following lower bound:
\begin{align}
\lambda (\NormI{\vts + \ds} - \NormI{\vts}) &\geq - \lambda \NormI{\ds} \notag \\
&\geq -\lambda \sqrt{|S|} \NormII{\ds} \notag \\
&= -\lambda \sqrt{|S|} b. \label{eq:lb2}
\end{align}
Now we turn our attention to computing a lower bound of the second term of \eqref{eq:f}, which
is a bit more involved.
\begin{align*}
\ds^T \Hess \loss(\vts + \theta \ds) \ds
&\geq \min_{\mathclap{\NormII{\ds} = b}} \ds^T \Hess \loss(\vts + \theta \ds) \ds \\
&= b^2 \eigmin(\Hess \loss(\vts + \theta \ds)).
\end{align*}
Now,
\begin{align*}
&\eigmin(\Hess \loss(\vts + \theta \ds)) \\
&\quad \geq \min_{\mathclap{\theta \in [0, 1]}} \eigmin\left(\Hess \loss(\vts + \theta \ds)\right) \\
&\quad =\min_{\mathclap{\theta \in [0, 1]}} \eigmin\Big(\frac{1}{m} \sum_{l=1}^m \eta((\vts + \theta \ds)^T \zsl) \zsl (\zsl)^T\Big)
\end{align*}
Again, using the Taylor's theorem to expand the function $\eta$ we get
\begin{align*}
&\eta((\vts + \theta \ds)^T \zsl) \\
&\qquad = \eta((\vts)^T \zsl) 
    + \eta'((\vts + \bar{\theta} \ds)^T \zsl)(\theta \ds)^T \zsl
\end{align*}
, where $\bar{\theta} \in [0, \theta]$. 
Finally, from Lemma \ref{lemma:eigvalue} we have,
with probability at least $1 - \Abs{S} \exp (\nicefrac{(-m \cmin)}{2 \Abs{S}})$:
\begin{align*}
&\eigmin\left(\Hess \loss(\vts + \theta \ds)\right) \\
&\;\geq \min_{\mathclap{\theta \in [0, 1]}} \eigmin \Big(\frac{1}{m} \sum_{l=1}^m \eta((\vts)^T \zsl) \zsl (\zsl)^T \\
    &\; + \frac{1}{m} \sum_{l=1}^m \eta'((\vts + \bar{\theta} \ds)^T \zsl)((\theta \ds)^T \zsl) \zsl (\zsl)^T \Big) \\
&\;\geq \eigmin(\mat{H}^m_{SS}) - \max_{\theta \in [0, 1]} \InNormII{\mat{A}(\theta)} \\
&\;\geq \frac{\cmin}{2} - \max_{\theta \in [0, 1]} \InNormII{\mat{A}(\theta)},
\end{align*}
where we have defined
\begin{align*}
\mat{A}(\theta) &\defeq \frac{1}{m} \sum_{l=1}^m \eta'((\vts + \theta \ds)^T \zsl) \times \\
&\quad (\theta \ds)^T \zsl \zsl (\zsl)^T.
\end{align*}
Next, the spectral norm of $\mat{A}(\theta)$ can be bounded as follows:
\begin{align*}
&\InNormII{\mat{A}(\theta)} \\
&\; \leq \max_{\mathclap{\NormII{\yv} = 1}} \Bigg\{
    \frac{1}{m} \sum_{l=1}^m \Abs{\eta'((\vts + \theta \ds)^T \zsl)} \Abs{((\theta \ds)^T \zsl)} \\
       &\qquad\qquad \times \yv^T (\zsl (\zsl)^T) \yv  \Bigg\} \\
&\; < \max_{\mathclap{\NormII{\yv} = 1}} \frac{1}{(10 m)} \sum_{l=1}^m \NormI{(\theta \ds)} \NormInfty{\zsl}
    \yv^T (\zsl (\zsl)^T) \yv  \\
&\; \leq \theta \max_{\NormII{\yv} = 1} \left\{ \frac{1}{10 m} \sum_{l=1}^m \sqrt{|S|} \NormII{\ds}
    \yv^T (\zsl (\zsl)^T) \yv \right\} \\
&\; = \theta b \sqrt{|S|} \InNormII{\frac{1}{10 m} \sum_{l=1}^m \zsl (\zsl)^T} \\
&\; \leq \frac{(b \sqrt{|S|} \dmax)}{5} \leq \frac{\cmin}{4},
\end{align*}
where in the second line we used the fact that $\eta'(.) < \nicefrac{1}{10}$
and in the last line we assumed that $\frac{(b \sqrt{|S|} \dmax)}{5} \leq \nicefrac{\cmin}{4}$ ---
an assumption that we verify momentarily. Having upper bounded the 
spectral norm of $\mat{A}(\theta)$, we have
\begin{align}
\eigmin\left(\Hess \loss(\vts + \theta \ds)\right) \geq \frac{\cmin}{4} \label{eq:lb3}.
\end{align}
Plugging back the bounds given by \eqref{eq:lb1}, \eqref{eq:lb2} and \eqref{eq:lb3}
in \eqref{eq:f} and equating to zero we get
\begin{gather*}
- b \sqrt{|S|} \left(\nu  + \sqrt{\frac{2}{m} \log \frac{2n}{\delta}}\right)
+ \frac{b^2 \cmin}{4} - \lambda \sqrt{\Abs{S}} b = 0 \\
\implies b = \frac{4 \sqrt{\Abs{S}}}{\cmin}
	\left(\lambda + \nu + \sqrt{\frac{2}{m} \log \frac{2n}{\delta}}\right).
\end{gather*} 
Finally, coming back to our prior assumption we have
\begin{align*}
b = \frac{4 \sqrt{\Abs{S}}}{\cmin}
	\left(\lambda + \nu + \sqrt{\frac{2}{m} \log \frac{2n}{\delta}}\right)
	\leq \frac{5 \cmin}{4 \sqrt{\Abs{S}} \dmax}.	
\end{align*}
The above assumption holds if the regularization parameter $\lambda$ is bounded
as follows:
\begin{align*}
\lambda \leq \frac{5\cmin^2}{16\Abs{S} \dmax} - \sqrt{\frac{2}{m} \log \frac{2n}{\delta}} - \nu.
\end{align*}
\end{proof}
%%%%%%%%%%%%%%%%%%%%%%%%%%%%%%%%%%%%%%%%%%%%%%%%%%%%%%%%%%%%%%%%%%%%%%%%%%%%%%%%%%%%%%%%%%%%%%%%%%
\begin{lemma}
\label{lemma_l1norm_bound}
If the regularization parameter $\lambda$ satisfies the following condition:
\begin{align*}
\lambda \geq \nu + \sqrt{\frac{2}{m} \log \frac{2n}{\delta}},
\end{align*}
then we have that 
\begin{align*}
\NormI{\vh - \vt} \leq \frac{5\cmin}{\dmax}
\end{align*}
with probability at least $1 - (\delta + \Abs{S} \exp (\nicefrac{(-m \cmin)}{2 \Abs{S}}) + 
\Abs{S} \exp ( - \nicefrac{m \tpmin}{4\Abs{S}} ) )$.
\end{lemma}
Now we are ready to present our main result on recovering the true PSNE set.
\begin{theorem}
\label{thm:psne}
If for all $i$, $\Abs{S_i} \leq k$, the minimum payoff $\rhomin \geq \nicefrac{5 \cmin}{\dmax}$,
and the regularization parameter and the number of samples satisfy the following conditions:
\begin{gather}
\nu  + \sqrt{\frac{2}{m} \log \frac{6n^2}{\delta}} \leq \lambda 
	\leq 2K + \nu - \sqrt{\frac{2}{m} \log \frac{6n^2}{\delta}} \label{eq:reg_param_bounds} \\
m \geq \max \Bigg\{ \frac{2}{K^2} \log \left(\frac{6n^2}{\delta}\right),	 
	\frac{2k}{\cmin} \log \left(\frac{3kn}{\delta} \right), \notag \\
		\frac{4k}{\tpmin} \log \left(\frac{3kn}{\delta} \right) \Bigg\}, \label{eq:sample_complexity}
\end{gather}
where $K \defeq \nicefrac{5\cmin^2}{32k\dmax} - \nu $, 
then with probability at least $1 - \delta$, for $\delta \in (0, 1)$,
we recover the true PSNE set, i.e., 
$\NE(\widehat{\mat{W}}, \widehat{\vect{b}}) = \NE(\mat{W}^*, \vect{b}^*)$.
\end{theorem}
\begin{proof}
From Cauchy-Schwartz inequality and Lemma \ref{lemma_l1norm_bound} we have
\begin{align*}
\Abs{(\vhi - \vti)^T \zi} \leq \NormI{\vhi - \vti} \NormInfty{\zi} \leq \frac{5 \cmin}{\dmax}.
\end{align*}
Therefore, we have that
\begin{align*}
(\vti)^T \zi - \frac{5 \cmin}{\dmax}
        \leq \vhi^T \zi \leq (\vti)^T \zi + \frac{5 \cmin}{\dmax}.
\end{align*}
Now, if $\forall\; \vect{x} \in \NE^*$, $(\vti)^T \zi \geq \nicefrac{5 \cmin}{\dmax}$,
then $\vhi^T \zi \geq 0$.  Using an union bound argument over all players $i$, we can
show that the above holds with probability at least  
\begin{align}
1 - n(\delta + k \exp (\nicefrac{(-m \cmin)}{2k}) + 
	k\exp (\nicefrac{(-m \tpmin)}{4k}) \label{eq:thm_main_prob}
\end{align}
for all players. 
Therefore, we have that $\NE(\widehat{\mat{W}}, \widehat{\vect{b}}) = \NE^*$ 
with high probability. Finally, setting $\delta = \nicefrac{\delta'}{3n}$,
for some $\delta' \in [0, 1]$, and ensuring that the last two terms in \eqref{eq:thm_main_prob}
are at most $\nicefrac{\delta'}{3n}$ each, we prove our claim.
\end{proof}
To better understand the implications of the theorem above, 
we instantiate it for the global and local noise model. 
\begin{remark}[Sample complexity under global noise model]
Recall that $\dmax \leq \min(k, 2^n \pmax)$, and for the global noise model given by \eqref{eq:obs_model_global}
$\pmax = \nicefrac{q_g}{\Abs{\NE^*}}$. If $q_g$ is constant, then $\dmax = k$.
Then $K = \BigOm{\nicefrac{1}{k^2}}$, and the sample complexity of learning sparse
linear games grows as $\BigO{k^4 \log n}$.
However, if $q_g$ is small enough, i.e., $q_g = \BigO{\nicefrac{\Abs{\NE^*}}{2^n}}$, then 
$\dmax$ is no longer a function of $k$ and $K = \BigOm{\nicefrac{1}{k}}$. 
Hence, the sample complexity scales as $\BigO{k^2 \log n}$ for exact PSNE recovery.
\end{remark}
\begin{comment}
The sample complexity of $\BigO{\poly(k) \log n}$ for exact recovery of the PSNE set
can be compared with the sample complexity of $\BigO{kn^3 \log^2 n}$ for the maximum likelihood estimate (MLE)
of the PSNE set as obtained by Honorio \cite{Honorio2016}. Note that while the MLE procedure is consistent,
i.e. MLE of the PSNE set is equal to the true PSNE set with probability converging to 1 as
the number of samples tend to infinity, it is NP-hard \footnote{Irfan and Ortiz \cite{irfan14}
showed that counting the number of Nash equilibria is $\#$P-complete. Therefore, computing the
log-likelihood is NP-hard.}.
In contrast, the logistic regression method
is computationally efficient. Further, while the sample complexity of our method
seems to be better than the empirical log-likelihood minimizer as given by Theorem 3 in \cite{Honorio2016},
in this paper we restrict ourselves to LIGs with strictly positive payoff in the PSNE set ---
such games are invariably easier to learn than general LIGs considered by \cite{Honorio2016}.
Honorio \cite{Honorio2016} also obtained lower bounds on the number of samples required by any conceivable 
method, for exact PSNE recovery, by scaling the parameter $q$ with the number of players.
\end{comment}

Next, we consider the implications of Theorem \ref{thm:psne} under the local noise model given by \eqref{eq:obs_model_local}.
we consider the regime where the parameter $q$ scales with the number of players $n$.
\begin{remark}[Sample complexity under local noise]
In the local noise model if the number of Nash-equilibria is constant, then
$\pmax = \BigO{exp(-n)}$, and once again $\dmax$ becomes independent of $k$, which results in a sample
complexity of $\BigO{k^2 \log n}$. 
\end{remark}

Also, observe the dependence of the sample complexity on the minimum noise level $\tpmin$. 
The number of samples required to recover the PSNE set increases as $\tpmin$ decreases.
From the aforementioned remarks we see that if the noise level is too low, i.e., $ \tpmin \rightarrow 0$, then
number of samples needed goes to infinity; This seems
counter-intuitive --- with reduced noise level, a learning problem should
become easier and not harder. To understand this seemingly counter-intuitive behavior,
first observe that the constant $\nicefrac{\dmax}{\cmin}$ can be thought of as the
``condition number'' of the loss function given by \eqref{eq:loss}. Then, the sample complexity
as given by Theorem \ref{thm:psne} can be written as $\BigO{\frac{k^2\dmax^2}{\cmin^2} \log n}$.
Hence, we have that as the noise level gets too low, the Hessian of the loss
becomes ill-conditioned, since the data set now comprises of many repetitions of the few
joint-actions that are in the PSNE set; thereby increasing the dependency ($\dmax$) between actions of players
in the sample data set.

\subsection{Necessary Conditions}
In this section we derive necessary conditions on
the number of samples required to learn graphical games.
Our approach for doing so is information-theoretic: we treat the 
inference procedure as a communication channel and then
use the Fano's inequality to lower bound the estimation error.
Such techniques have been widely used to obtain necessary conditions
for model selection in graphical models, see e.g. \cite{santhanam2012information,Wang2010},
sparsity recovery in linear regression \cite{Wainwright2009}, and many other problems.

Consider an ensemble $\fG_n$ of $n$-player games with the 
in-degree of each player being at most $k$. Nature picks a true game $\Game^* \in \fG$,
and then generates a data set $\Data$ of $m$ joint actions. A decoder is any function
$\psi: \X^m \rightarrow \fG_n$ that maps a data set $\Data$ to a game, $\psi(\Data)$, in $\fG_n$. The
minimum estimation error over all decoders $\psi$, for the ensemble $\fG_n$, is then given as follows:
\begin{align}
\perr \defeq \min_{\psi} \max_{\Game^* \in \fG_n} \Prob{\NE(\psi(\Data)) \neq \NE(\Game^*)},
\end{align}
where the probability is computed over the data distribution. Our objective here
is to compute the number of samples below which PSNE recovery fails with probability
greater than $\nicefrac{1}{2}$.
 
\begin{theorem}
\label{thm:fano}
The number of samples required to learn graphical games over $n$ players 
and in-degree of at most $k$, is $\BigOm{k \log n}$.
\end{theorem}
\begin{remark}
From the above theorem and from Theorem \ref{thm:psne} we observe that the method of $l_1$-regularized 
logistic regression for learning graphical games, 
operates close to the fundamental limit of $\BigOm{k \log n}$.
\end{remark}%
Results from simulation experiments for both global and local noise model can be
found in Appendix \ref{sec:experiments}.
\paragraph{Concluding Remarks.}
\label{sec:conclusion}
%In this paper, we presented a computationally efficient and statistically consistent 
%method, based on $\ell_1$-regularized logistic regression, for learning linear influence games --- 
%a subclass of parametric graphical games with linear payoffs. Under some mild conditions
%on the true game, we showed that as long as the number of samples scales
%as $\BigO{\poly(k) \log n}$, where $n$ is the number of players and $k$ is
%the maximum number of neighbors of any player; then we can recover
%the pure-strategy Nash equilibria set of the true game in polynomial time and with probability
%converging to 1 as the number of samples tend to infinity. 
An interesting direction for future work
would be to consider structured actions --- for instance
permutations, directed spanning trees, directed acyclic graphs among others --- thereby extending
the formalism of linear influence games to the structured prediction setting. 
Other ideas that might be worth pursuing are: considering mixed strategies,
correlated equilibria and epsilon Nash equilibria, and incorporating latent or unobserved
actions and variables in the model.

\begin{small}
\bibliographystyle{IEEEtran}
\bibliography{paper}
\end{small}
\clearpage
\begin{appendices}
\section{Detailed Proofs}
\label{app:detailed_proofs}
%%%%%%%%%%%%%%%%% PROOFS %%%%%%%%%%%%%%
%%%%%%%%%%% Number of Nash equilibria %%%%%%
\begin{proof}[Proof of Proposition \ref{prop:number_of_equilibria}]
Let $\Abs{\NE^*} > 2^{n - 1}$. Then by the pigeon hole principle there are
at least two joint actions $\x$ and $\x'$ in $\NE^*$ such that $\x = -\x'$.
Since the payoff is strictly positive, it follows that the bias $b_i$ for 
each player must be $0$. If the bias for all players is $0$, then for each
$\x \in \NE^*$, $-\x \in \NE^*$. Therefore, $\Abs{\NE^*} = 2^n$. Since we
have assumed that the game is non-trivial, we get a contradiction.
\end{proof}
%%%%%%%%%%%%%%%%% L1 Norm %%%%%%%%%%%%%
\begin{proof}[Proof of Lemma \ref{lemma_l1norm_bound}]
Define $\Delta \defeq \vh - \vt$. Also for any vector $\vect{y}$ let
the notation $\vect{y}_{\overline{S}}$ denote the vector $\vect{y}$ with
the entries not in the support, $S$, set to zero, i.e.
\begin{align*}
\left[\vect{y}_{\overline{S}} \right]_i &= \left\{\begin{array}{lr}
y_i & \text{if $i \in S$},\\
0 & \text{otherwise}.
\end{array}\right.
\end{align*}
Similarly, let the notation $\vect{y}_{\overline{\Sc}}$ denote the vector $\vect{y}$
with the entries not in $\Sc$ set to zero, where $\Sc$ is the complement of $S$.
Having introduced our notation and since, $S$ is the support of the true vector $\vt$,
we have by definition that $\vt = \vtbs$. We then have, using the reverse triangle inequality,
\begin{align}
\NormI{\vh} &= \NormI{\vt + \Delta} = \NormI{\vtbs + \dbs + \dbsc} \notag \\
&= \NormI{\vtbs - (- \dbs)} + \NormI{\dbsc} \notag \\
&\geq \NormI{\vt} - \NormI{\dbs} + \NormI{\dbsc}. \label{eq:onenorm_bound1}
\end{align}
Also, from the optimality of $\vh$ for the $\ell_1$-regularized problem we have that
\begin{align}
\loss(\vt) + \lambda \NormI{\vt} &\geq \loss(\vh) + \lambda \NormI{\vh} \notag \\
\implies \lambda(\NormI{\vt} - \NormI{\vh}) &\geq \loss(\vh) - \loss(\vt). \label{eq:onenorm_bound2}
\end{align}
Next, from convexity of $\loss(.)$ and using the Cauchy-Schwartz inequality
we have that
\begin{align}
\loss(\vh) - \loss(\vt) &\geq \Grad \loss(\vt)^T(\vh - \vt) \notag \\
&\geq - \NormInfty{\Grad \loss(\vt)} \NormI{\Delta} \notag \\
&\geq - \frac{\lambda}{2} \NormI{\Delta}, \label{eq:onenorm_bound3}
\end{align}
in the last line we used the fact that $\lambda \geq \NormInfty{\Grad \loss(\vt)}$.
Thus, we have from \eqref{eq:onenorm_bound1},
\eqref{eq:onenorm_bound2} and \eqref{eq:onenorm_bound3} that 
\begin{align}
& \frac{1}{2} \Norm{\Delta}_1 \geq \NormI{\vh} - \NormI{\vt} \notag \\
\implies &\frac{1}{2} \Norm{\Delta}_1 \geq \NormI{\dbsc} - \NormI{\dbs} \notag \\
\implies &\frac{1}{2} \NormI{\dbsc} + \frac{1}{2} \NormI{\dbs} \geq \NormI{\dbsc} - \NormI{\dbs} \notag \\
\implies &3 \NormI{\dbs} \geq \NormI{\dbsc}. \label{eq:onenorm_bound4}
\end{align}
Finally, from \eqref{eq:onenorm_bound4} and Lemma \ref{lemma_l2norm_bound} we
have that
\begin{align*}
\NormI{\Delta} &= \NormI{\dbs} + \NormI{\dbsc} \\
&\leq 4 \NormI{\dbs} \leq 4 \sqrt{\Abs{S}} \NormII{\dbs} \\
&\leq \frac{5\cmin}{\dmax}.
\end{align*}
\end{proof}
%%%%%%%% FANO %%%%%%%%%%%
\begin{proof}[Proof of Theorem \ref{thm:fano}]
First, we construct a restricted ensemble of games $\ftG \subset \fG$ as follows. Each game $\Game \in \ftG$
contains $k$, randomly chosen, \emph{influential} players. The game graph for $\Game$
is then chosen to be a complete directed bipartite graph from the set of $k$ influential
players to the set of $n-k$ non-influential players. The edge weights are all set to 
$-1$, the bias for the $k$ influential players is set to $+1$, while the bias for the
remaining $n-k$ players is set to $0$. Then it is clear that \emph{each game in $\ftG$ induces
a distinct size-one PSNE set}. Specifically, for a game $\Game \in \ftG$,
a joint action $\x \in \NE(\Game)$ is such that $x_i = -1$ if player $i$ is influential
in $\Game$, otherwise $x_i = +1$. Also, note that the minimum payoff in the PSNE set
of each game in $\ftG$ is strictly positive, and is precisely $1$. Finally, we assume
that the data set is drawn according to the global noise model $\eqref{eq:obs_model_global}$,
with $q = \nicefrac{1}{n}$.
Now let $\Game \in \ftG$ be a uniformly-distributed 
random variable corresponding to the game that was picked by nature.
From the Fano's inequality, we have that:
\begin{align}
\perr \geq 1 - \frac{\MI(\Data ; \Game) + \log 2}{H(\Game)}, \label{eq:fano}
\end{align}
where $\MI(.)$ denotes mutual information and $H(.)$ denotes entropy.
Since, $\Game$ is uniformly distributed, we have that 
$H(\Game) = \log \big |\ftG \big| = \log {n \choose k} \geq k (\log n - \log k)$.
Let $\Pf_{\Data | \Game = \Game_1}$ be the conditional distribution of the data set given a game $\Game_1 \in \ftG$.
We bound the mutual information $\MI(\Data; \Game)$ by a pairwise KL-based bound from 
\cite{Yu97} as follows:
\begin{align}
\MI(\Data ; \Game) \leq 
	\frac{1}{ \big | \ftG \big|} \sum_{\Game_1 \in \ftG} \sum_{\Game_2 \in \ftG} 
		\KL{\Pf_{\Data | \Game = \Game_1}}{\Pf_{\Data | \Game = \Game_2}}. \label{eq:mi}
\end{align}
Now from the fact that data are sampled \iid, we get:
\begin{align}
&\textstyle {\KL{\Pf_{\Data | \Game = \Game_1}}{\Pf_{\Data | \Game = \Game_2}}}  \notag \\
	&\quad \textstyle{= m \sum_{\x \in \X} \Pf_{\Data | \Game = \Game_1}(\x) 
		\log \frac{\Pf_{\Data | \Game = \Game_1}(\x)}{\Pf_{\Data | \Game = \Game_2}(\x)}} \notag \\
	&\quad = \textstyle{m \left\{ q \log \frac{q(2^n - 1)}{1 - q} + \frac{1 - q}{2^n - 1} \log \frac{1 - q}{q(2^n - 1)} \right\}} \notag \\
	&\quad = \textstyle{\frac{m (2^n q - 1)}{2^n-1} \left(\log{q} - \log{\left(\frac{1-q}{2^n-1}\right)} \right)} \notag \\
	&\quad \leq m \log 2, \label{eq:kl}
\end{align}
where the last line comes from the fact that $q = \nicefrac{1}{n}$.
Putting together \eqref{eq:fano}, \eqref{eq:mi} and \eqref{eq:kl}, and setting $\perr = \nicefrac{1}{2}$,
we get
\begin{align*}
m \leq \frac{k \log n - k \log k - 2 \log 2}{2 \log 2}.
\end{align*}
By observing that learning the ensemble $\fG$ is at least as hard as learning a subset of the
ensemble $\ftG$, we prove our main claim.
\end{proof}

\section{Experiments}
\label{sec:experiments}
In order to verify that our results and assumptions indeed hold in practice,
we performed various simulation experiments. We generated random LIGs for $n$ players
and exactly $k$ neighbors by first creating a matrix $\mat{W}$ of all zeros and then
setting $k$ off-diagonal entries of each row, chosen uniformly at random, to $-1$.
We set the bias for all players to $0$.
We found that any odd value of $k$ produces games with strictly
positive payoff in the PSNE set. Therefore, for each value of $k$ in $\{1, 3, 5\}$,
and $n$ in $\{10, 12, 15, 20\}$, we generated $40$ random LIGs. For experiments involving
the local noise model, we only used $n \in \{10, 12, 15\}$. The parameter 
$\delta$ was set to the constant value of $0.01$. For the global noise model,
the parameters $q_g$ was set to $0.01$, while for the local noise model we used
$q_1 = \ldots = q_n = 0.6$. The regularization parameter $\lambda$ was set according to Theorem \ref{thm:psne}
as some constant multiple of $\sqrt{(\nicefrac{2}{m}) \log (\nicefrac{2n}{\delta})}$.
Figure \ref{fig:simulation} shows the probability of successful recovery of the PSNE,
for various combinations of $(n, k)$, where the probability was computed
as the fraction of the $40$ randomly sampled LIGs for which the learned PSNE set matched the
true PSNE set exactly.
For each experiment, the number of samples was computed as:
$\lfloor(C) (10^{c}) (k^2 \log(\nicefrac{6n^2}{\delta}))\rfloor$, where $c$ is the control parameter 
and the constant $C$ is $10000$ for $k = 1$ and $1000$ for $k = 3$ and $5$. Thus,
from Figure \ref{fig:simulation} we see that, the sample complexity of $\BigO{k^2 \log n}$ as given
by Theorem \ref{thm:psne} indeed holds in practice, i.e., there exists constants
$c$ and $c'$ such that if the number of samples is less than $c k^2 \log n$, we
fail to recover the PSNE set exactly with high probability, while if the number of 
samples is greater than $c' k^2 \log n$ then we are able to recover the PSNE set
exactly, with high probability. Further, the scaling remains consistent as the
number of players $n$ is changed from $10$ to $20$.
\begin{figure*}[htbp]
\centering
\includegraphics[width=0.7\linewidth]{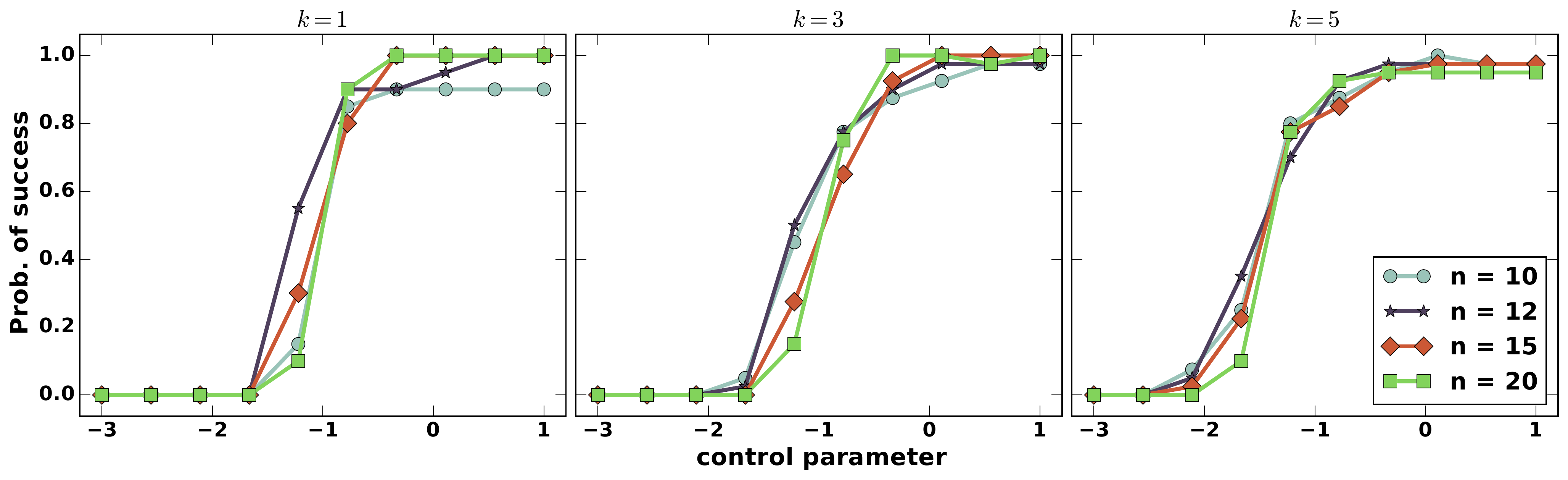}
\includegraphics[width=0.7\linewidth]{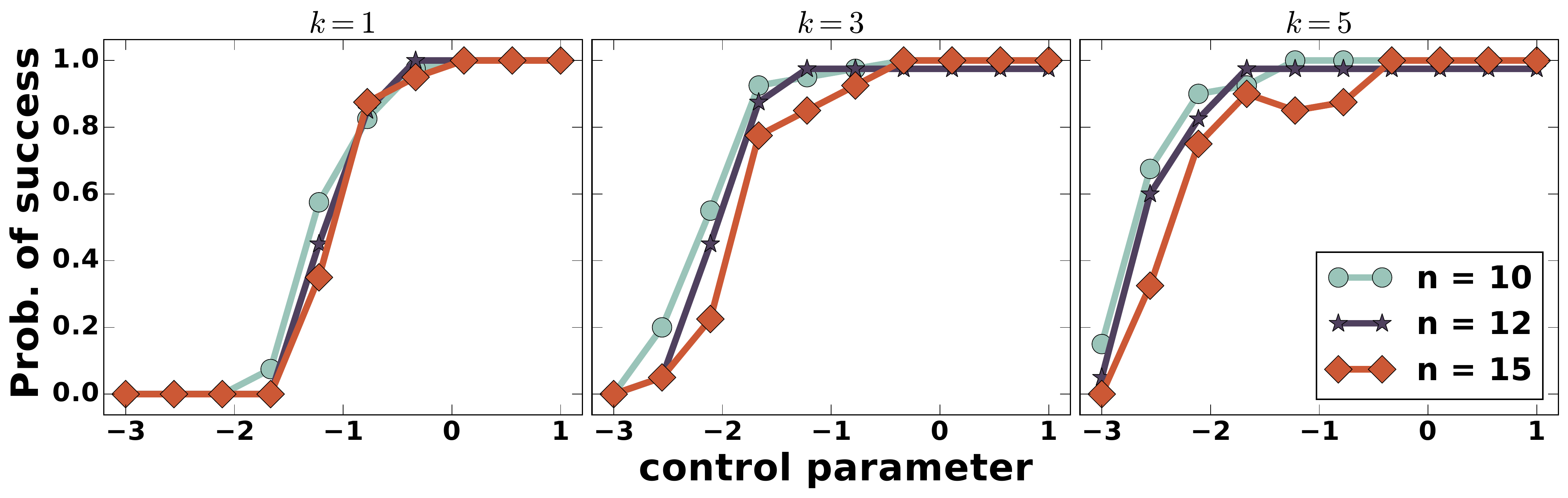}
\caption{The probability of exact recovery of the PSNE set computed across $40$ randomly sampled LIGs,
using the global noise model (TOP) and local noise model (BOTTOM),
as the number of samples is scaled as $\lfloor (C) (10^{c}) (k^2 \log(\nicefrac{6n^2}{\delta}))\rfloor$,
where $c$ is the control parameter and the constant $C$ is $10000$ for the $k = 1$ case 
and $1000$ for the remaining two case. For the global noise model we set $q_g = 0.001$, while
for the local noise model we used $q_1 = \ldots = q_n = 0.6$.%
\label{fig:simulation}}
\end{figure*}%

\end{appendices}

\end{document}